\documentclass[twoside]{article}
\usepackage[accepted]{aistats2016}
\usepackage{amsmath}
\usepackage{amssymb}
\usepackage{amsthm}
\usepackage{algorithm}
\usepackage{algpseudocode}
\usepackage{graphicx}
\usepackage{color}

\DeclareMathOperator\argmin{arg\,min}
\DeclareMathOperator\argmax{arg\,max}
\DeclareMathOperator\sgn{sgn}
\DeclareMathOperator\pr{Pr}

\newtheorem{thm}{Theorem}

\newtheorem{defi}{Definition}
\newtheorem{prop}{Proposition}
\newtheorem{remk}{Remark}

\begin{document}


\runningauthor{Jincheng Mei, Hao Zhang, Bao-Liang Lu}

\twocolumn[

\aistatstitle{On the Reducibility of Submodular Functions}

\aistatsauthor{ Jincheng Mei \And Hao Zhang \And Bao-Liang Lu$^\star$ }

\aistatsaddress{ Department of Computing Science \\ University of Alberta \\ Edmonton, AB, Canada, T6G 2E8 \\ \And The Robotics Institute \\ School of Computer Science \\ Carnegie Mellon University \\ Pittsburgh, PA 15213, USA \And Key Lab of SMEC for IICE \\ Dept. of Computer Sci. and Eng. \\ Shanghai Jiao Tong University \\ Shanghai 200240, China} ]

\begin{abstract}
The scalability of submodular optimization methods is critical for their usability in practice. In this paper, we study the \emph{reducibility} of submodular functions, a property that enables us to reduce the solution space of submodular optimization problems without performance loss. We introduce the concept of reducibility using \emph{marginal gains}. Then we show that by adding perturbation, we can endow irreducible functions with reducibility, based on which we propose the \emph{perturbation-reduction} optimization framework. Our theoretical analysis proves that given the perturbation scales, the reducibility gain could be computed, and the performance loss has additive upper bounds. We further conduct empirical studies and the results demonstrate that our proposed framework significantly accelerates existing optimization methods for irreducible submodular functions with a cost of only small performance losses.
\end{abstract}

\section{INTRODUCTION}
\label{sec:intro}
Submodularity naturally arises in a number of machine learning problems, such as active learning \cite{golovin2011adaptive}, clustering \cite{narasimhan2005q}, and dictionary selection \cite{das2011submodular}. The scalability of submodular optimization methods is critical in practice, thus has drawn much attention from the research community. For example, Iyer et al. \cite{iyer2013fast} propose $\mathcal{O}(n^2)$ general optimization methods based on the semidifferential. Wei et al. \cite{wei2014fast} combine approximation with pruning to accelerate the greedy algorithm for uniform matroid constrained submodular maximization. Mirzasoleiman et al. \cite{mirzasoleiman2013distributed} and Pan et al. \cite{pan2014parallel} use distributed implementation to accelerate existing optimization methods. Other techniques, including stochastic sampling \cite{mirzasoleiman2013lazier} and decomposable assumption \cite{jegelka2013reflection}, are also applied to scale up submodular optimization methods.
While in this paper, we focus on the reducibility of submodular functions, a favourable property that can substantially improve the scalability of submodular optimization methods. The reducibility can directly reduce the solution space of the submodular optimization problems, while preserve all the optima in the reduced space, thereby enables us to accelerate the optimization process without incurring performance loss.

Recent research shows that for some submodular functions, by evaluating marginal gains, reduction can be applied for unconstrained maximization \cite{goldengorin2009maximization}, unconstrained minimization \cite{fujishige2005submodular,iyer2013fast}, and uniform matroid constrained maximization \cite{wei2014fast}. By leveraging the reducibility, a variety of methods have been developed to scale up the optimization of \emph{reducible} submodular functions \cite{fujishige2005submodular,goldengorin2009maximization,iyer2013fast,mei2015unconstrained},
While existing works mainly focus on reducible functions, there exist a number of \emph{irreducible} submodular functions widely applied in practice, for which existing methods can only provide vacuous reduction.

In this paper, we investigate the problem that whether irreducible functions can also exploit this favorable property.
We firstly introduce the concept of reducibility using marginal gains over the endpoint sets of a given lattice. Then for irreducible functions, we transform them to reducible functions by adding random noise to perturb the marginal gains, after which we perform lattice reduction for the perturbed functions and solve the original functions on the reduced lattice. Theoretical results show that given the perturbation scales, the reducibility gain is lower bounded, and the performance loss has additive upper bounds. The empirical results demonstrate that there exist useful perturbation scale intervals in practice, which enables us to significantly accelerate existing optimization methods with small performance losses.

In summary, this paper has the following contributions. Firstly, we introduce the concept of reducibility, and propose the perturbation-reduction framework. Secondly, we theoretically analyze our proposed method. In particular, for the reducibility gain, we propose a lower bound in terms of the perturbation scale. For the performance loss, we propose both deterministic and probabilistic upper bounds. The deterministic bound provides the understanding of relationship between the reducibility gain and performance loss, while the probabilistic bound can explain the experimental results. Finally, we empirically show that the proposed method is applicable for a variety of commonly used irreducible submodular functions.

In the sequel, we organize the paper as follows. In Section \ref{sec:reducibility}, we introduce the definitions and the existing reduction algorithms. In Section \ref{sec:perturbation}, we propose our perturbation based method. Theoretical analysis and empirical results are presented in Section \ref{sec:theory} and Section \ref{sec:results}, respectively. In Section \ref{sec:related_work}, we review some related works. Section \ref{sec:conclusion} comes to our conclusion.

\section{REDUCIBILITY}
\label{sec:reducibility}
\subsection{Notations and Definitions}

Given a finite set $N = \{ 1, 2, \dots, n\}$, and a set function $f: 2^N \mapsto \mathbb{R}$. $f$ is said to be \emph{submodular} \cite{fujishige2005submodular} if $\forall X, Y \subseteq N$, $f(X) + f(Y) \ge f(X \cap Y) + f(X \cup Y)$. An equivalent definition of submodularity is the \emph{diminishing return} property: $\forall A \subseteq B \subseteq N$, $\forall i \in N \setminus B$, $f(i | A) \ge f(i | B)$, where $f(i | A) \triangleq f(A + i) - f(A)$ is called the \emph{marginal gain} of element $i$ with respect to set $A$. To simplify the notation, we denote $A \cup \{ i \}$ by $A + i$, and $A \setminus \{ i \}$ by $A - i$. Given $ A \subseteq B \subseteq N$, the \emph{(set interval) lattice} is defined as $[A,B] \triangleq \{ S \ | \ A \subseteq S \subseteq B \}$.

Suppose $N = \{ 1, 2, \dots, n\}$, and $f: 2^N \mapsto \mathbb{R}$ is a submodular function. In this paper, we focus on unconstrained submodular optimization problems,
\begin{equation*}
    \text{Problem 1}: \ \min_{X \subseteq N}{f(X)}, \quad \text{Problem 2}: \ \max_{X \subseteq N}{f(X)}.
\end{equation*}
Problem 1 can be exactly solved in polynomial time \cite{orlin2009faster}, while Problem 2 is NP-hard since some of its special cases (e.g., Max Cut) are NP-hard.

For convenience of the presentation, we will use P1 and P2 to refer to Problem 1 and Problem 2, respectively. Similarly, for the following algorithms, we will use A1 and A2 to refer to Algorithm 1 and Algorithm 2. The reference holds for P3, P4, A3 and A4.

Define $\mathcal{X}_{min} \triangleq \{ X_* \subseteq N \ |\ f(X_*) \le f(X), \forall X \subseteq N \}$ as the optima set of P1. Similarly, define $\mathcal{X}_{max} \triangleq \{ X^* \subseteq N \ |\ f(X^*) \ge f(X), \forall X \subseteq N \}$. Obviously, we have $\mathcal{X}_{min} \subseteq [\emptyset,N]$ and $\mathcal{X}_{max} \subseteq [\emptyset,N]$.

We now give the definition of reducibility. For P1 (P2), we say the objective function $f$ is \emph{reducible for minimization (maximization)} if $\exists \ [S,T] \subset [\emptyset,N]$, where $[S,T]$ can be obtained in $\mathcal{O}(n^p)$ function evaluations, such that $\mathcal{X}_{min} \subseteq [S,T]$ ($\mathcal{X}_{max} \subseteq [S,T]$). Note that if we can only find $[S,T]$ in $\mathcal{O}(2^n)$ time, the reduction is meaningless since $\mathcal{O}(2^n)$ time is enough for us to find all the optima. The ratio $1- \frac{|T \setminus S|}{|N|} \in (0,1]$ is called the \emph{reduction rate}.

\subsection{Algorithms}

Existing works on reduction for unconstrained submodular optimization can be summarized by the following two algorithms, both of which terminate in $\mathcal{O}(n^2)$ time. The brief review of existing works can be found in Section 6.

\begin{minipage}[h]{0.48\textwidth}
\begin{algorithm}[H]
\caption{Reduction for Minimization}
\label{algo1}
\begin{algorithmic}[1]
\Require
$N$, $f$, $X_0 \leftarrow \emptyset$, $Y_0 \leftarrow N$, $t \leftarrow 0$.
\Ensure
$[X_t,Y_t]$.
\State Find $U_t = \{i \in Y_t \setminus X_t \ |\  f(i|X_t) < 0\}$.
\State $X_{t+1} \leftarrow X_t \cup U_t$.
\State Find $D_t = \{j \in Y_t \setminus X_t \ |\ f(j|Y_t - j) > 0\}$.
\State $Y_{t+1} \leftarrow Y_t \setminus D_t$.
\State If $X_{t+1} = X_t$ and $Y_{t+1} = Y_t$, terminate.
\State $t \leftarrow t+1$. Go to Step $1$.
\end{algorithmic}
\end{algorithm}
\end{minipage}
\begin{minipage}[h]{0.48\textwidth}
\begin{algorithm}[H]
\caption{Reduction for Maximization}
\label{algo2}
\begin{algorithmic}[1]
\Require
$N$, $f$, $X_0 \leftarrow \emptyset$, $Y_0 \leftarrow N$, $t \leftarrow 0$.
\Ensure
$[X_t,Y_t]$.
\State Find $U_t = \{i \in Y_t \setminus X_t \ |\  f(i|X_t) < 0\}$.
\State $Y_{t+1} \leftarrow Y_t \setminus U_t$.
\State Find $D_t = \{j \in Y_t \setminus X_t \ |\ f(j|Y_t - j) > 0\}$.
\State $X_{t+1} \leftarrow X_t \cup D_t$.
\State If $X_{t+1} = X_t$ and $Y_{t+1} = Y_t$, terminate.
\State $t \leftarrow t+1$. Go to Step $1$.
\end{algorithmic}
\end{algorithm}
\end{minipage}

\begin{prop}
\label{prop1}
Suppose $f: 2^N \mapsto \mathbb{R}$ is submodular. After each iteration of A1 (A2), we have $\mathcal{X}_{min} \subseteq [X_t,Y_t]$ ($\mathcal{X}_{max} \subseteq [X_t,Y_t]$).
\end{prop}

We prove Proposition \ref{prop1} in the supplementary material. According to Proposition \ref{prop1}, if the output of A1 (A2) statifies $[X_t,Y_t] \subset [\emptyset,N]$, then $f$ is reducible.

According to A1 (A2), if $U_0 = D_0 = \emptyset$, then we have $X_1 = X_0$ and $Y_1 = Y_0$. The algorithm will terminate after the first iteration and the output is $[X_0,Y_0] = [\emptyset,N]$, which provides a vacuous reduction. In this case, we say that $f$ is \emph{irreducible with respect to A1 (A2)}. For convenience, we directly say $f$ is \emph{irreducible}.

Thereby, we conclude two points from the above algorithms. First, by the definition of $U_t$ and $D_t$, the reducibility of $f$ can be determined by the signs of marginal gains with respect to the endpoint sets of the current working lattice. Second, the reducibility of $f$ for minimization and maximization are actually the same property. Specially, suppose in a certain iteration, A1 and A2 have the same working lattice $[S,T]$. According to the algorithms, they also have the same $U_t$ and $D_t$, which determine whether $f$ is reducible after the current iteration.

\begin{prop}
\label{prop2}
Given a submodular function $f: 2^N \mapsto \mathbb{R}$, and a lattice $[S,T]$. $\forall i \in T \setminus S$, Define $K_i = \sgn\{f(i|S)\} \cdot \sgn\{f(i|T-i)\}$. Then $f$ is reducible on $[S,T]$ with respect to A1 (A2) if and only if
\begin{equation}
\label{eq1}
K = \max_{i \in T \setminus S}{K_i} > 0.
\end{equation}
\end{prop}
\begin{proof}
Suppose $[X_0,Y_0] = [S,T]$ in A1 (A2). Then $f$ is reducible if and only if the algorithm does not terminate after its first iteration, \emph{i.e.}, $U_0 \not= \emptyset$ or $D_0 \not= \emptyset$. Suppose $U_0 \not= \emptyset$ happens, \emph{i.e.}, $\exists i \in T \setminus S$, $f(i|S) < 0$. According to submodularity, $f(i|T-i) \le f(i|S) < 0$. We have $K \ge K_i = \sgn\{f(i|S)\} \cdot \sgn\{f(i|T-i)\} = 1 > 0$. Suppose $D_0 \not= \emptyset$ happens, \emph{i.e.}, $\exists j \in T \setminus S$, $f(j|T-j) > 0$. According to submodularity, $f(j|S) \ge f(j|T-j) > 0$. We have $K \ge K_j = \sgn\{f(j|S)\} \cdot \sgn\{f(j|T-j)\} = 1 > 0$.
\end{proof}

According to Proposition \ref{prop2}, the reducibility of $f$ for minimization (maximization) can be obtained by (\ref{eq1}). Thus we say $f$ is \emph{reducible with respect to A1 (A2)} if (\ref{eq1}) holds. Similarly, without ambiguity in this paper, we directly say $f$ is \emph{reducible} if (\ref{eq1}) holds.

\section{PERTURBATION REDUCTION}
\label{sec:perturbation}
Given a reducible submodular function, we can use A1 and A2 to provide useful reduction. Unfortunately, there still exist many irreducible submodular functions, some of which are listed in the experimental section. Given a submodular function $f$, which is irreducible on $[S,T]$. According to Proposition \ref{prop2}, $\forall i \in T \setminus S$, we have $f(i|S) \ge 0$ and $f(i|T-i) \le 0$.

If we expect A1 and A2 to provide nontrivial reduction, we need to guarantee that (\ref{eq1}) holds for some elements without changing the submodularity of the objective function. A natural way is to add random noise $r$\footnote{$r: 2^N \mapsto \mathbb{R}$ is a modular function, and $r(X) \triangleq \sum\limits_{i \in X}{r(i)}$.} to perturb the original function as follows,
\begin{align*}
    &\text{Problem 3}: \ \min_{X \subseteq N}{g(X)} \triangleq \min_{X \subseteq N}{f(X) + r(X)}, \\
    &\text{Problem 4}: \ \max_{X \subseteq N}{g(X)} \triangleq \max_{X \subseteq N}{f(X) + r(X)},
\end{align*}
where $\forall i \in N$, $r(i) \in \mathbb{R}$ is generated uniformly at random in $[-t,t]$ for some $t \ge 0$. By appropriately choosing the value of $t$, we can ensure $g(i|S) < 0$ or $g(i|T-i) > 0$ hold for some $i \in T \setminus S$. Thus we have (\ref{eq1}) holds, indicating that $g$ is reducible. At the same time, as $r$ is a modular function, the submodularity of $g$ still holds.

\begin{minipage}[h]{0.48\textwidth}
\begin{algorithm}[H]
\caption{Perturbation-Reduction Minimization}
\label{algo3}
\begin{algorithmic}[1]
\Require
$N$, $f$, $[S,T]$ where $\mathcal{X}_{min} \subseteq [S,T]$.
\Ensure An approximate solution $X_*^p$.
\State If $f$ is reducible on $[S,T]$, $[X_0,Y_0] \leftarrow [S,T]$, run A1 for $f$, $[S,T] \leftarrow [X_t,Y_t]$.
\State Generate $r$. Let $g = f + r$. $[X_0,Y_0] \leftarrow [S,T]$, run A1 for $g$, $[S,T] \leftarrow [X_t,Y_t]$.
\State Solve $X_*^p \in \argmin_{X \in [S,T]}{f(X)}$.
\end{algorithmic}
\end{algorithm}
\end{minipage}
\begin{minipage}[h]{0.48\textwidth}
\begin{algorithm}[H]
\caption{Perturbation-Reduction Maximization}
\label{algo4}
\begin{algorithmic}[1]
\Require
$N$, $f$, $[S,T]$ where $\mathcal{X}_{max} \subseteq [S,T]$.
\Ensure An approximate solution $X_p^*$.
\State If $f$ is reducible on $[S,T]$, $[X_0,Y_0] \leftarrow [S,T]$, run A2 for $f$, $[S,T] \leftarrow [X_t,Y_t]$.
\State Generate $r$. Let $g = f + r$. $[X_0,Y_0] \leftarrow [S,T]$, run A2 for $g$, $[S,T] \leftarrow [X_t,Y_t]$.
\State Solve $X_p^* \in \argmax_{X \in [S,T]}{f(X)}$.
\end{algorithmic}
\end{algorithm}
\end{minipage}

We propose our perturbation based method for minimization and maximization in A3 and A4, respectively. For an irreducible submodular function $f$ on a given lattice $[S,T]$, we first perturb the objective function to make it reducible, \emph{i.e.}, $g \triangleq f + r$. A1 or A2 are then employed to obtain the reduced lattice of $g$. Finally we solve the original problems of $f$ on the reduced lattice exactly or approximately using existing methods.

It is worth mentioning that, though we mainly focus on irreducible functions, our methods also work for reducible ones, as they are special cases of irreducible functions. Particularly, given a reducible function $f$ on $[S,T]$, of which the reduction rate is less than $1$, after A1 (A2) terminates, we can get a sublattice $[P,Q] \subset [S,T]$ so that $f$ is irreducible on $[P,Q]$.

\section{THEORETICAL ANALYSIS}
\label{sec:theory}
By perturbing the irreducible submodular function, we transform P1 (P2) into P3 (P4). This makes the objective reducible while leads the solution to be inexact. Correspondingly, our theoretical analysis of the method focuses on two main aspects: the reducibility gain and the performance loss incurred by perturbation.

\subsection{Reducibility Gain}
\label{reducibility_gain}

Suppose $f: 2^N \mapsto \mathbb{R}$ is an irreducible submodular function on $[S,T]$, and $g \triangleq f + r$ as defined in P3 and P4. Since $f$ is irreducible, $\forall i \in T \setminus S$, we have $f(i|S) \ge 0$ and $f(i|T-i) \le 0$.

\begin{prop}
\label{prop3}
Given a submodular function $f: 2^N \mapsto \mathbb{R}$, which is irreducible on $[S,T]$. Define $m\{f,[S,T]\} \triangleq \min_{i \in T \setminus S}{\min\{f(i|S), -f(i|T-i)\}}$. If $t \le m\{f,[S,T]\}$, then $g$ is irreducible on $[S,T]$.
\end{prop}
\begin{proof}
Since $m\{f,[S,T]\} \ge 0$, we suppose $0 \le t \le m\{f,[S,T]\}$. $\forall i \in T \setminus S$, we have $g(i|S) = f(i|S) + r(i) \ge f(i|S) - t \ge f(i|S) - m\{f,[S,T]\} \ge 0$, and $g(i|T-i) = f(i|T-i) + r(i) \le f(i|T-i) + t \le f(i|T-i) + m\{f,[S,T]\} \le 0$, which implies that $g$ is also irreducible on $[S,T]$.
\end{proof}

Proposition \ref{prop3} indicates that if the perturbation scale $t$ is small enough, there is no reducibility gain. This is intuitively reasonable since we have $g \rightarrow f$ when $t \rightarrow 0$.

To lower bound the reducibility gain of adding perturbation, we generalize the concept of \emph{curvature} \cite{conforti1984submodular,iyer2013curvature} for non-monotone irreducible submodular functions.

\begin{defi}
\label{defi1}
Given a submodular function $f: 2^N \mapsto \mathbb{R}$, the curvature of $f$ on $[S,T]$ is defined as,
\begin{equation*}
    c\{f,[S,T]\} = \max_{i \in T \setminus S, f(i|S) > 0}{\frac{f(i|S)-f(i|T-i)}{f(i|S)}}.
\end{equation*}
\end{defi}

Note that for any irreducible submodular function $f$ on $[S,T]$, we have $c\{f,[S,T]\} \ge 1$.

\begin{thm}
\label{thm1}
Suppose $t > m\{f,[S,T]\}$, denote $s = |T \setminus S| > 0$, $k = \sum_{i \in T \setminus S}{f(i|S)}$, $c = c\{f,[S,T]\}$. The reduction rate in expectation of $g$ is at least $1 - \frac{ck}{2ts}$.
\end{thm}
\begin{proof}
Suppose $T \setminus S = \{ 1, 2, \dots, s \}$, $\forall i \in T \setminus S$, we define a random variable $H_i$ as,
\begin{equation*}
    H_i=
    \begin{cases}
        1 &\text{if } K_i > 0, \\
        0 &\text{otherwise}.
    \end{cases}
\end{equation*}
$H_i$ indicates whether $i$ can be reduced from $T \setminus S$ or not. Define $H = \sum_{i \in T \setminus S}{H_i}$ as the total number of the reduced elements. We firstly lower bound $\mathbb{E}(H)$ by the total number of the reduced elements after the first iteration round of A1(A2),
\begin{align*}
\small
    &\mathbb{E}(H) = \sum_{i \in T \setminus S}{\mathbb{E}(H_i)} = \sum_{i \in T \setminus S} {\pr\{H_i = 1\}}\\
    &\ge \frac{1}{2t} \cdot \sum_{i \in T \setminus S}{\max\{0, t - f(i|S)\} + \max\{0, t + f(i|T-i)\}} \\
    &\ge s - \frac{c}{2t} \cdot \sum_{i \in T \setminus S}{f(i|S)} = s - \frac{ck}{2t}.
\end{align*}
Consequently, the reduction rate in expectation is $\frac{\mathbb{E}(H)}{s} \ge 1 - \frac{ck}{2ts}$.
\end{proof}
Theorem \ref{thm1} implies that the reduction rate in expectation approaches $1$ as the perturbation scale $t$ increases. This is also consistent with our intuition since $g \rightarrow r$ when $t \rightarrow \infty$. Note that $r$ is a modular function, which always has the highest reduction rate $1$.

\subsection{Performance Loss}

Suppose $X_* \in \mathcal{X}_{min}$ ($X^* \in \mathcal{X}_{max}$), \emph{i.e.}, $X_*$ ($X^*$) is an optimum of P1 (P2). Recall that $X_*^p$ ($X_p^* $) is the output of A3 (A4), for P1 (P2), we define $f(X_*^p) - f(X_*)$ ($f(X^*) - f(X_p^*)$) as the \emph{performance loss} incurred by perturbation. For P1 (P2), the following result shows that the performance loss is upper bounded by the total perturbation of the ``mistakenly'' reduced elements, which will be explained later on.

\begin{thm}
\label{thm2}
Given an irreducible submodular function $f: 2^N \mapsto \mathbb{R}$. Suppose $t$ is the perturbation scale in A3 (A4), and $R_t$ is the reduction rate. We have,
{
\vspace{-3pt}
\begin{align*}
\small
  f(X_*^p) - f(X_*) &< - r(X_t \setminus X_*) + r(X_* \setminus Y_t) < n t R_t, \\
  f(X^*) - f(X_p^*) &< r(X_t \setminus X^*) - r(X^* \setminus Y_t) < n t R_t.
\end{align*}
}
\end{thm}
\vspace{-12pt}
\begin{proof}
We prove the maximization case. In general, we have $X^* \not\in [X_t, Y_t]$, otherwise the loss is zero. Note that $X_p^* \in [X_t, Y_t]$ according to A4. So we firstly introduce an intermediate set $X^* \cup X_t \cap Y_t$, \emph{i.e.}, the contraction of $X^*$ in $[X_t,Y_t]$ for our analysis. Given the fact that $f(X^* \cup X_t \cap Y_t) \le f(X_p^*)$, if we can upper bound $f(X^*) - f(X^* \cup X_t \cap Y_t)$, then the total performance loss is also upper bounded.
In A2, we have $[X_t,Y_t] \subset \cdots \subset [X_1,Y_1] \subset [X_0,Y_0] = [\emptyset,N]$. By definition, $\forall \ 0 \le k \le t-1$, $X_{k+1} = X_k \cup D_k$, and $Y_{k+1} = Y_k \setminus U_k$. We have,
\begin{align}
\small
\label{eq2}
  &f(X^* \cup X_t) - f(X^*) \\
\label{eq3}
  &= \sum\limits_{s = 1, x_s \in X_t \setminus X^*}^{|X_t \setminus X^*|}{f(x_s | X^* + x_1 + \cdots + x_{s-1})} \\
\label{eq4}
  &\ge \sum\limits_{i = 0}^{t-1} \sum\limits_{d \in D_i \setminus X^*}{f(d|Y_i-d)} \\
\label{eq5}
  &> - \sum\limits_{i = 0}^{t-1} \sum\limits_{d \in D_i \setminus X^*}{r(d)} \\
\label{eq6}
  &= - r(X_t \setminus X^*),
\end{align}
where (\ref{eq3}) is the telescopic version of (\ref{eq2}). According to submodularity, we have (\ref{eq4}) holds, and (\ref{eq5}) comes from the third step of A2. Similarly, we have,
\begin{equation}
\label{eq7}
  f(X^* \cup X_t) - f(X^* \cup X_t \cap Y_t) < -r(X^* \setminus Y_t).
\end{equation}
Combining (\ref{eq6}) with (\ref{eq7}), and noting $X^* \cup X_t \cap Y_t \in [X_t, Y_t]$ and $f(X_p^*) \ge f(X^* \cup X_t \cap Y_t)$, we have,
\begin{equation}
\label{eq8}
  f(X^*) - f(X_p^*) < r(X_t \setminus X^*) - r(X^* \setminus Y_t).
\end{equation}
We note in (\ref{eq8}), $X_t \setminus X^*$ is actually the set of all the elements which are not in $X^*$ but added by A2. Similarly, $X^* \setminus Y_t$ is the set of all the elements which are in $X^*$ but eliminated by A2. Consequently, the performance loss is upper bounded by the total perturbation value of all the mistakenly reduced elements. Since the number of all the mistakenly reduced elements is no more than the number of all the reduced elements $n R_t$, and the perturbation is generated in $[-t, t]$, we have $r(X_t \setminus X^*) - r(X^* \setminus Y_t) \le n t R_t$.

For the minimization case, the proof is similar.
\end{proof}

Note that in Theorem \ref{thm2}, the performance loss is upper bounded by the sum of random variables, which means we can obtain high probability bounds using some concentration inequalities, such as \cite{hoeffding1963probability}.

\begin{thm}(\textbf{Hoeffding})
\label{thm3}
Let $X_1, X_2, \dots, X_n$ be independent real-valued random variables such that $\forall i \in \{ 1, 2, \dots, n\}$, $|X_i| \le t$. Then with probability $1 - \delta$,
\begin{equation*}
\small
    \sum_{i=1}^n{X_i} - \mathbb{E}\left[\sum_{i=1}^n{X_i}\right] <  t \sqrt{2n \log{(1/\delta)}}.
\end{equation*}
\end{thm}

\begin{thm}
\label{thm4}
Define $X_*^c \triangleq X_* \cup X_t \cap Y_t$, and $X_c^* \triangleq X^* \cup X_t \cap Y_t$. Denote $M_r \triangleq |X_*^c \triangle X_*|$, and  $N_r \triangleq |X_c^* \triangle X^*|$, where $A \triangle B \triangleq (A \setminus B) \cup (B \setminus A)$ is the symmetric difference between the two sets $A$ and $B$. Then with probability at least $1 - \delta$,
\begin{align}
\small
\label{eq9}
    f(X_*^p) - f(X_*) &< t \sqrt{2 M_r (n + \log{(1/\delta)})}, \\
\label{eq10}
    f(X^*) - f(X_p^*) &< t \sqrt{2 N_r (n + \log{(1/\delta)})}.
\end{align}
\end{thm}
\begin{proof}
We prove (\ref{eq10}). Since the perturbation vector $r$ has zero expectation value, and each element of $r$ is independently generated. For any fixed $X \subseteq N$, according to Theorem \ref{thm3},  with probability at least $1 - \delta$,
\begin{equation}
\label{eq11}
    r(X) - r(X^*) \le t \sqrt{2 |X \triangle X^*| \log{(1/\delta)}}.
\end{equation}
Suppose $X^* \in [S,T]$, and define $m \triangleq |[S,T]|$. Obviously, we have $m \le 2^n$. Hence,
\begin{equation*}
\footnotesize
\begin{aligned}
    &\pr{\left[ r(X_c^*) - r(X^*) \ge t \sqrt{2 N_r \log{(m/\delta)}} \right]} \\
    &= \hspace{-8pt} \sum\limits_{X \in [S,T]} \hspace{-5pt} \pr{\left[ r(X_c^*) - r(X^*) \hspace{-2pt} \ge \hspace{-2pt} t \hspace{-1pt} \sqrt{2 |X_c^* \triangle X^*| \log{(\frac{m}{\delta})}}, X_c^* \hspace{-2pt}  = \hspace{-2pt}  X \right]} \\
    &= \hspace{-8pt} \sum\limits_{X \in [S,T]} \hspace{-5pt} \pr{\left[ r(X) - r(X^*) \hspace{-2pt} \ge \hspace{-2pt} t \sqrt{2 |X \triangle X^*| \log{(\frac{m}{\delta})}}, X_c^* \hspace{-2pt} = \hspace{-2pt} X \right]} \\
    &\le \hspace{-8pt} \sum\limits_{X \in [S,T]} \hspace{-5pt} \pr{\left[ r(X) - r(X^*) \hspace{-2pt} \ge \hspace{-2pt} t \sqrt{2 |X \triangle X^*| \log{(\frac{m}{\delta})}} \right]} \\
    &\le \hspace{-8pt} \sum\limits_{X \in [S,T]} \frac{\delta}{m} = m \frac{\delta}{m} = \delta,
\end{aligned}
\end{equation*}
where the first equality holds by the law of total probability. The second equality holds because replacing $X_c^*$ with $X$ in the first expression does not change the event. The first inequality comes from dropping the event $X_c^* = X$ increases the probability. The last line results from (\ref{eq11}) and the definition of $m$.
Combining the above result with Theorem \ref{thm2}, and note $r(X_t \setminus X^*) - r(X^* \setminus Y_t) = r(X_c^*) - r(X^*)$, we have, with probability at least $1 - \delta$,
\vspace{-2pt}
\begin{equation*}
\footnotesize
    f(X^*) - f(X_p^*) \hspace{-2pt} < \hspace{-2pt} r(X_c^*) - r(X^*) \hspace{-2pt} < \hspace{-2pt} t \sqrt{2 N_r (n + \log{(1/\delta)})}.
\vspace{-2pt}
\end{equation*}
Using a similar method, (\ref{eq9}) can also be proved.
\end{proof}

Theorem \ref{thm4} has an intuitive interpretation. Take P2 and P4 as examples, $\forall Y \subseteq N$, if $f(X^*) - f(Y)$ is large, then it is unlikely that $Y$ is an optimum of P4. Suppose $f(X^*) - f(Y) = \sigma > 0$, then we have,
\begin{equation*}
\footnotesize
\begin{aligned}
    &\pr{\left[f(Y) + r(Y) \ge f(X) + r(X), \forall X \subseteq N\right]} \\
    &\le \pr{\left[f(Y) + r(Y) \ge f(X^*) + r(X^*)\right]} \\
    &= \pr{\left[r(Y) - r(X^*) \ge \sigma \right]}.
\end{aligned}
\vspace{-2pt}
\end{equation*}
Totally, the probability of $Y$ being an optimum of P4 is upper bounded by the probability that the perturbation difference $r(Y) - r(X^*)$ can compensate the function value difference $\sigma$, where the later probability is small when $\sigma$ is large.

Finally, we show that the $M_r$ and $N_r$ in Theorem \ref{thm4}, which are the numbers of mistakenly reduced elements, can also be upper bounded by functions of $f$ and $t$.

\begin{thm}
\label{thm5}
Denote the total number of the mistakenly reduced elements in the first iteration of A1 and A2 as $M^1_r$ and $N^1_r$,  respectively.
We have,
{\footnotesize
\begin{align}
\label{eq12}
    \mathbb{E}_r[M^1_r] &\le \frac{n}{2} - \frac{F - f(X_*)}{t}, \\
\label{eq13}
    \mathbb{E}_r[N^1_r] &\le \frac{n}{2} - \frac{f(X^*) - F}{t},
\end{align}
}
where $F \triangleq \frac{1}{2}(f(\emptyset) + f(N)) \in [f(X_*), f(X^*)]$.
\end{thm}
\begin{proof}
For (\ref{eq13}), we calculate the total mistakenly reduced element number in expectation in the first iteration of A2. According to the definition of symmetric difference, $N^1_r = |X^* \setminus Y_1| + |X_1 \setminus X^*|$.

$\forall i \in X^*$, $f(i | \emptyset) \ge f(i | X^* - i) = f(X^*) - f(X^* - i) \ge 0$. And $i \in X^* \setminus Y_1$ iff $f(i | \emptyset) + t < 0$. Similarly, $\forall j \not\in X^*$, $f(j | N - j) \le f(j | X^*) = f(X^* + j) - f(X^*) \le 0$. And $j \in X_1 \setminus X^*$ iff $f(j | N - j) + t > 0$. Thus we have,
\begin{equation*}
\footnotesize
\begin{aligned}
    &\mathbb{E}_r[N^1_r] = \mathbb{E}_r|X^* \setminus Y_1| + \mathbb{E}_r|X_1 \setminus X^*| \\
    &=\sum_{i \in X^*}{\frac{t - f(i | \emptyset)}{2t}} + \sum_{j \not\in X^*}{\frac{t + f(j | N - j)}{2t}} \\
    &=\frac{n}{2} - \frac{1}{2t}\left[ \sum_{i \in X^*}{f(i | \emptyset)} - \sum_{j \not\in X^*}{f(j | N - j)} \right] \\
    &\le \frac{n}{2} - \frac{f(X^*) - f(\emptyset) + f(X^*) - f(N)}{2t}.
\end{aligned}
\end{equation*}
For (\ref{eq12}), the proof is similar.
\end{proof}

Using similar methods we can obtain the following results, which recover Theorem \ref{thm5} as a special case.

\begin{thm}
\label{thm6}
Denote the total number of the mistakenly reduced elements in the $k$th iteration as $M^k_r$, and  $N^k_r$,  respectively. We have,
\begin{align*}
    \mathbb{E}_r[M^k_r] &\le \frac{n_{k-1}}{2} - \frac{F_{k-1} - f(X_*)}{t}, \\
    \mathbb{E}_r[N^k_r] &\le \frac{n_{k-1}}{2} - \frac{f(X^*) - F_{k-1}}{t},
\end{align*}
where $n_{k-1} \triangleq |Y_{k-1} \setminus X_{k-1}|$, and $F_{k-1} \triangleq \frac{1}{2}(f(X_{k-1}) + f(Y_{k-1})) \in [f(X_*), f(X^*)]$.
\end{thm}

Theorem \ref{thm5} implies that the expected number of mistakenly reduced elements in the first iteration will approach $\frac{n}{2}$ as the perturbation scale $t$ increases. This is consistent with the intuition. Let $t \rightarrow \infty$, then $g \rightarrow r$. Each element will be randomly selected to be added or eliminated with probability $\frac{1}{2}$, so the expected number of mistakenly reduced elements is $\frac{n}{2}$.

\begin{remk}
\label{remk1}
When $t$ is large enough, most elements will be reduced in the first iteration, i.e., $N_r \approx N^1_r$. Let $t =  \frac{2(f(X^*) - F)}{n(1-2\varepsilon)}$, where $\varepsilon > 0$, by (\ref{eq13}) we have $\mathbb{E}_r[N_r] \approx \mathbb{E}_r[N^1_r] \le \varepsilon n$, which indicates that if $t = \frac{2(f(X^*) - F)}{n(1-2\varepsilon)}$ is a large enough perturbation scale, then the number of  mistakenly reduced elements can be desirably upper bounded.
\end{remk}

\begin{remk}
\label{remk2}
Suppose $t$ is large enough and $N_r \approx N^1_r$. With the result of Theorem \ref{thm2}, we have
\begin{equation*}
    f(X^*) - f(X_p^*) < tN_r \approx \frac{nt}{2} - (f(X^*) - F).
\end{equation*}
Let $t = \frac{2[(1+\delta)f(X^*) - F]}{n}$ where $\delta > 0$, we have $f(X_p^*) > (1-\delta)f(X^*)$ from above. This means if there is some relationship between the optimum and the perturbation ($t = \frac{2[(1+\delta)f(X^*) - F]}{n}$ is a large perturbation scale), then the previous performance loss results can be transformed into approximation ratios.
\end{remk}

\begin{figure*}[t]
\centering
\includegraphics[width=1\linewidth]{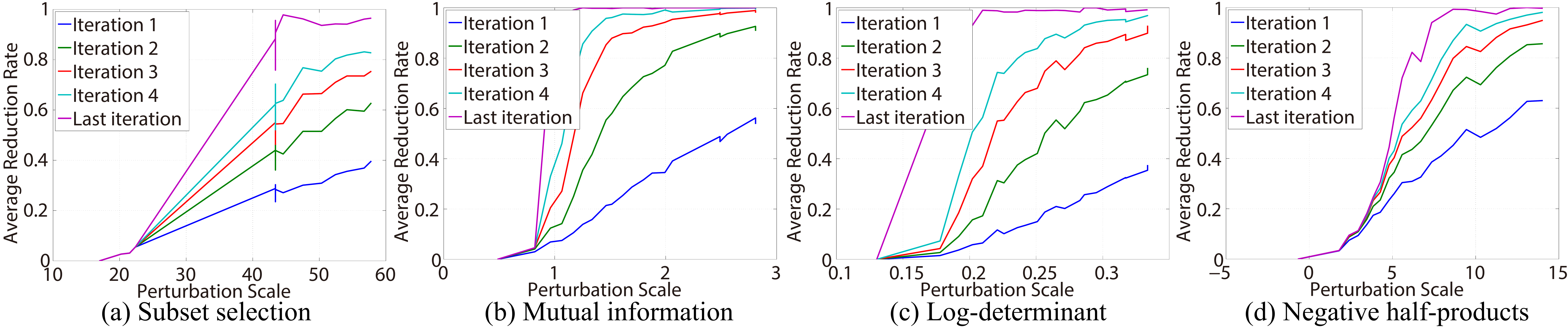}
\caption{Average Reduction Rates of Maximization}
\label{fig1}
\end{figure*}

\section{EXPERIMENTAL RESULTS}
\label{sec:results}
For reducible submodular functions, by incorporating reduction into optimization methods, favorable performance has been achieved \cite{fujishige2005submodular,goldengorin2009maximization,iyer2013fast,mei2015unconstrained}. In our experiments, we mainly focus on (nearly) irreducible submodular functions, as listed below.

\paragraph{Subset Selection Function.}

The objective function \cite{lin2009select,iyer2013fast} is irreducible. Given $M \in \mathbb{R}_+^{n \times n}$, $f(X) \triangleq \sum_{i \in N}\sum_{j \in X}{M_{ij}} - \lambda\sum_{i,j \in X}{M_{ij}}$, where $\lambda \in [0.5,1]$. We set $n = 100$, $\lambda = 0.7$, and randomly generate symmetric matrix $M$ in $(0,1)^{n \times n}$, and set $M_{ii} = 1$, $\forall i \in N$.

\paragraph{Mutual Information Function.}

Given $n$ random vectors $X_1, X_2, \dots, X_n$, define $h(X)$ as the entropy of random variables $\{X_i | i \in N\}$, which is a highly reducible submodular function. The symmetrization \cite{bach2013learning} of $h$ leads to the mutual information $f(X) \triangleq h(X) + h(N \setminus X)$, which is irreducible. We set $n = 100$, and randomly generate $\{X_i \ | \ i = 1, 2, \dots, n\}$.

\paragraph{Log-Determinant Function.}

Given a positive definite matrix $K \in \mathbb{S}_{++}^n$, the determinant \cite{kulesza2012determinantal} is log-submodular. The symmetrization of log-determinant is $f(X) \triangleq \log{\det{(K_X)}} + \log{\det{(K_{N \setminus X})}}$, where $K_X \triangleq [K_{ij}]_{i,j \in X}$, $\forall X \subseteq N$. We set $n = 100$. We randomly generate $n$ data points and compute the $n \times n$ similarity matrix as the positive definite matrix $K$.

\paragraph{Negative Half-Products Function.}

The objective \cite{boros2002pseudo} is $f(X) \triangleq c(X) - \sum_{i,j \in X, i<j}{a(i)b(j)}$, where $a, b, c$ are non-negative vertors. When $c$ is not non-negative, $f$ can be highly reducible \cite{mei2015unconstrained}. Here $c$ is non-negative, and $f$ is nearly irreducible. The reduction rate of A1 (A2) is about $1\%$. We set $n = 100$, and randomly generate $a, b$ in $(0.1,0.5)^n$ and $c$ in $(1,5)^n$.

\begin{figure*}[t]
\centering
\includegraphics[width=1\linewidth]{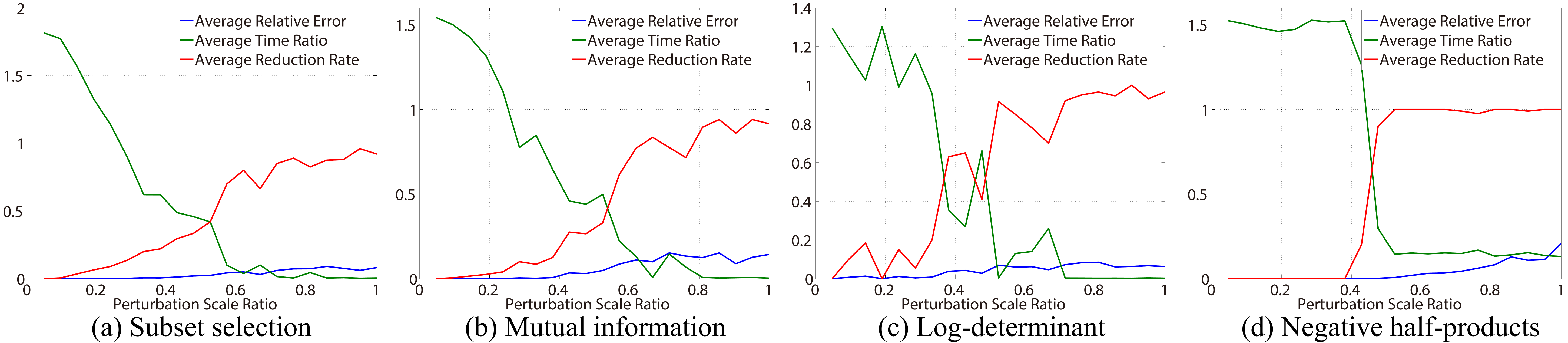}
\caption{Maximization Results Using Branch-and-Bound Method \cite{goldengorin1999data} (Exact Solver)}
\label{fig2}
\end{figure*}

\begin{figure*}[t]
\centering
\includegraphics[width=1.0\linewidth]{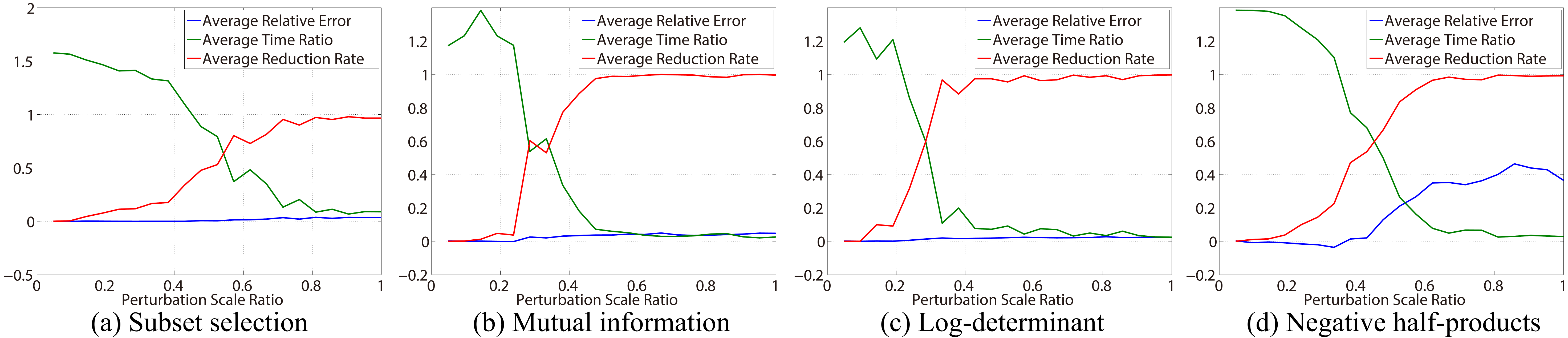}
\caption{Maximization Results Using Random Bi-directional Greedy \cite{buchbinder2012tight} (Approximate Solver)}
\label{fig3}
\end{figure*}

\subsection{Perturbation Scale}

In Theorem \ref{thm1}, we lower bound the expectation of reduction rate using the expectation of reduction rate after the first iteration of A1 (A2). It is recently reported that for reducible submodular minimization, this bound is often not tight in practice \cite{iyer2013fast}. Given that our method is actually transforming irreducible functions to reducible ones, it is reasonable to borrow experience from reducible cases. We conjecture that relatively small reduction rates after the first iteration would be sufficient for high reduction rates after the last iteration, thereby we only need to choose small perturbation scales to obtain desirable reducibility gains.
We empirically verify the conjecture as shown in Figure \ref{fig1}. Appropriate perturbation scales $t$ are chosen so that the reduction rates after the last iteration are changing from $0$ to nearly $1$. Given a certain perturbation scale, we repeatedly generate $r$ for $10$ times and record the average reduction rates of A2 after iteration $1$-$4$ and the last iteration. We observe that A2 terminates within $10$ iterations for all objective functions.

We defer similar results for minimization to the supplementary material. As conjectured, we learn from Figure \ref{fig1} that the gap between the average reduction rates after the first iteration and the last iteration is always large in practice. Hence, we can choose $t$ to get appropriate reduction rates in expectation (e.g., $0.3$) after the first iteration, so as to obtain potentially high final reduction rates.
Although we can empirically utilize the gap of reducibility gain to choose relative small perturbation scales, we would like to point out that theoretically determining the reduction rates in expectation after the last iteration given certain perturbation scales is still an open problem.

\subsection{Optimization Results}

We implement our method using SFO toolbox \cite{krause2010sfo}.

For maximization, we compare A4 with both exact and approximate methods, as exact methods usually cannot terminate in acceptable time with larger input scales. Denote the outputs of the proposed A4 and the existing method as $X^p$ and $X^e$, respectively. Also denote the running time as $T_p$ and $T_e$. We measure the performance loss using \emph{relative error}, which is defined as $E_r \triangleq \frac{|f(X^e)-f(X^p)|}{|f(X^e)|}$. When $X^e$ is exact, $1 - E_r$ is the approximation ratio. We measure the reducibility gain using both the reduction rate and the time ratio $T_p/T_e$. Small time ratios and relative errors indicate large reducibility gains and small performance losses, respectively.

We employ the branch-and-bound method \cite{goldengorin1999data} as the exact solver. Since it has exponential time complexity, we reset $n = 20$ so that it terminates within acceptable time. The results are shown in Figure \ref{fig2}. For comparison, we normalize the perturbation scale as follows. We define $M\{f,[S,T]\} \triangleq \max_{i \in T \setminus S}{\max\{f(i|S), -f(i|T-i)\}}$, and define the \emph{perturbation scale ratio} as $P(t) \triangleq \frac{t-m}{M-m}$. We change the perturbation scale $t$ in $[m,M]$ by varying $P(t)$ in $[0,1]$. We then randomly generate $10$ cases for each objective function and record the average relative errors, average reduction rates, and average time ratios for each perturbation scale ratio.
Figure \ref{fig3} shows the results compared with the random bi-directional greedy method \cite{buchbinder2012tight}, which is used as the approximate solver. Note that $n$ is set to 100. For each case, we firstly run A2 once, and then run the random method $5$ times on both the original and the reduced lattice, and record the best solutions.

According to Figure \ref{fig2} and Figure \ref{fig3}, when the perturbation scale ratio is smaller than $0.3$, the time ratio is larger than $1$. This is because the small reducibility gain cannot make the combination methods more efficient than before. As the perturbation scale ratio increases, the reduction rate increases and the time ratio decreases as expected. Meanwhile, the relative error increases gently when $P(t)$ increases, indicating that there exist useful intervals, in which the perturbation scales can lead to large reducibility gains and small performance losses.

For minimization, since the subset selection and the mutual information function have trivial zero optimal values, \emph{i.e.}, $f(X_*) = f(\emptyset) = 0$, we use the later two as objective functions. We employ the Fujishige-Wolfe minimum-norm point algorithm \cite{fujishige2011submodular} as the exact solver. All the settings are the same as those of maximization. The results of minimization are shown in Figure \ref{fig4}. We note that for negative half-products function, the useful interval of perturbation scales is smaller than those of other functions. According to Remark \ref{remk2}, as the marginal gains are relatively large compared to the optimal value, it is inappropriate to choose large perturbation scales in this case.

\begin{figure}[h]
\centering
\includegraphics[width=1.0\linewidth]{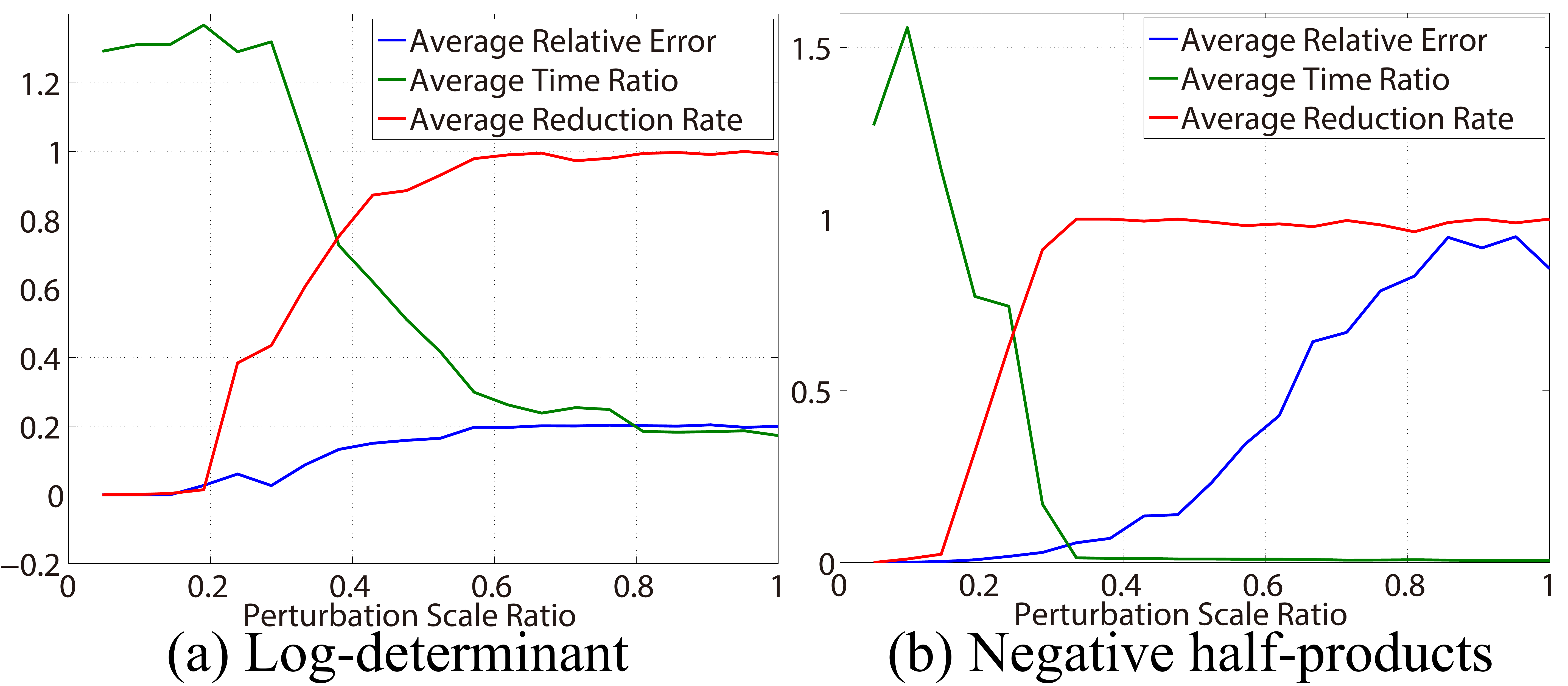}
\caption{Minimization Results}
\label{fig4}
\end{figure}

\section{RELATED WORK}
\label{sec:related_work}
In this section, we review some existing works related to solution space reduction for submodular optimization. For P1, Fujishige \cite{fujishige2005submodular} firstly proves $\mathcal{X}_{min} \subseteq [A,B]$, where $A = \{ i \in N \ | \ f(i|\emptyset) < 0 \}$ and $B = \{ j \in N \ | \ f(j|N-j) \le 0\}$. Note that actually $[A,B] = [X_1,Y_1]$, which is the working lattice of A1 after its first iteration. Recently, Iyer et al. \cite{iyer2013fast} propose the discrete Majorization-Minimization (MMin) framework for P1. They prove that by choosing appropriate supergradients, MMin is identical with A1. For P2, Goldengorin \cite{goldengorin2009maximization} proposes the Preliminary Preservation Algorithm (PPA), which is identical with A2. For general cases, Mei et al. \cite{mei2015unconstrained} prove that the two algorithms work for quasi-submodular functions. Beyond unconstrained problems, for uniform matroid constrained monotone submodular function optimization, Wei et al. \cite{wei2014fast} propose similar pruning method in which the reduced ground set contains all the original solutions of the \emph{greedy algorithm}.

\section{CONCLUSIONS}
\label{sec:conclusion}
In this paper, we introduce the reducibility of submodularity, which can improve the efficiency of submodular optimization methods. We then propose the perturbation-reduction framework, and demonstrate its advantages theoretically and empirically. We analyze the reducibility gain and performance loss given perturbation scales. Experimental results show that there exists practically useful intervals, and choosing perturbation scales from them enables us to significantly accelerate the existing methods with only small performance loss. For the future work, we would like to study the reducibility of submodular functions in constrained problems.

\subsubsection*{Acknowledgements}

Jincheng Mei would like to thank Csaba Szepesv{\'a}ri for fixing the proof of Theorem \ref{thm4}. Bao-Liang Lu was supported by the National Basic Research Program of China (No. 2013CB329401), the National Natural Science Foundation of China (No. 61272248) and the Science and Technology Commission of Shanghai Municipality (No. 13511500200). Asterisk indicates the corresponding author.

{\small
\bibliographystyle{plain}
\bibliography{aistats2016bib}
}

\newpage
\appendix
\section*{Appendix}
\section{Proof of Proposition \ref{prop1}}
For Algorithm \ref{algo1}, the proof can be found in \cite{iyer2013fast}. For Algorithm \ref{algo2}, the proof can be found in \cite{goldengorin2009maximization}. A proof using weaker assumption of quasi-submodular function $f$ can be found in \cite{mei2015unconstrained}. We prove Proposition \ref{prop1} here for completeness.

\begin{proof}
\textbf{Algorithm \ref{algo1}}. Obviously $\mathcal{X}_{min} \subseteq [X_0,Y_0]$. Suppose $\mathcal{X}_{min} \subseteq [X_k,Y_k]$, we now prove $\mathcal{X}_{min} \subseteq [X_{k+1},Y_{k+1}]$. Suppose $X_* \in \mathcal{X}_{min}$ is a minimum of $f$, then we have $X_k \subseteq X_* \subseteq Y_k$. For $\forall i \in U_k$, if $i \not\in X_*$, by submodularity, we have $f(i|X_*) \le f(i|X_k) < 0$, \emph{i.e.}, $f(X_*+i) < f(X_*)$, which contradicts with the optimality of $X_*$. So we have $U_k \subseteq X_*$, and $X_{k+1} = X_k \cup U_k \subseteq X_*$.
$\forall j \in D_k$, if $j \in X_*$, by submodularity, we have $f(j|X_*-j) \ge f(j|Y_k-j) > 0$, \emph{i.e.}, $f(X_*) > f(X_*-j)$, which also contradicts with the optimality of $X_*$. Therefore we have $D_k \subseteq N \setminus X_*$, and $X_* \subseteq Y_{k+1} = Y_k \setminus D_k$.

Now we have $X_{k+1} \subseteq X_* \subseteq Y_{k+1}$. Since $X_*$ can be an arbitrary element of $\mathcal{X}_{min}$, we have $\mathcal{X}_{min} \subseteq [X_{k+1},Y_{k+1}]$.

\textbf{Algorithm \ref{algo2}}. Obviously $\mathcal{X}_{max} \subseteq [X_0,Y_0]$. Suppose $\mathcal{X}_{max} \subseteq [X_k,Y_k]$, we now prove $\mathcal{X}_{max} \subseteq [X_{k+1},Y_{k+1}]$. Suppose $X^* \in \mathcal{X}_{max}$ is a maximum of $f$, then we have $X_k \subseteq X^* \subseteq Y_k$. $\forall i \in U_k$, if $i \in X^*$, by submodularity, we have $f(i|X^*-i) \le f(i|X_k) < 0$, \emph{i.e.}, $f(X^*) < f(X^*-i)$, which contradicts with the optimality of $X^*$. So we have $U_k \subseteq N \setminus X^*$, and $X^* \subseteq Y_{k+1} = Y_k \setminus U_k$.
$\forall j \in D_k$, if $j \not\in X^*$, by submodularity, we have $f(j|X^*) \ge f(j|Y_k-j) > 0$, \emph{i.e.}, $f(X^*+j) > f(X^*)$, which also contradicts with the optimality of $X^*$. So we have $D_k \subseteq X^*$, and $X_{k+1} = X_k \cup D_k \subseteq X^*$.

Now we have $X_{k+1} \subseteq X^* \subseteq Y_{k+1}$. Since $X^*$ can be an arbitrary element of $\mathcal{X}_{max}$, we have $\mathcal{X}_{max} \subseteq [X_{k+1},Y_{k+1}]$.
\end{proof}

\section{Reduction Rate of Algorithm \ref{algo1}}

Figure \ref{fig1_supp} shows the reduction rates of Algorithm \ref{algo1}. All the settings are the same as those of Algorithm \ref{algo2} in the paper.

\section{More Experimental Results}

\subsection{Results of Maximization}

In the paper we use the random bi-directional greedy method as the approximate solver for maximization. We also report the results of random permutation \cite{iyer2013fast} and random local search \cite{iyer2013fast}. The settings are the same as those in the paper. The results are shown in Figure \ref{fig2_supp} and Figure \ref{fig3_supp}.

\subsection{Results Using Real Data}

Finally, we compare the results on real data. The objective function is the log-determinant function. For each test case, we randomly select $100$ samples from the CIFAR dataset \cite{krizhevsky2009learning}, and then we compute the similarity matrix as the positive definite matrix $K$. Other settings are the same as those in the paper. The results are shown in Figure \ref{fig4_supp}.

In Figure \ref{fig4_supp}, the first three subfigures show the results of maximization using random local search, random permutation, and random bi-directional greedy, respectively. The last subfigure presents the results of minimization using the Fujishige-Wolfe minimum-norm point algorithm \cite{fujishige2011submodular}.

\begin{figure*}[t]
\centering
\includegraphics[width=1\linewidth]{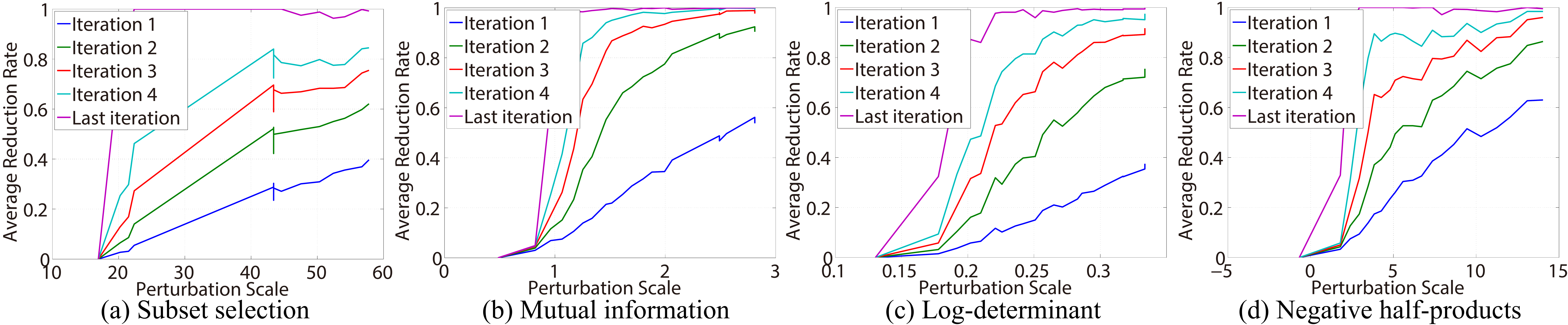}
\caption{Average Reduction Rates of Minimization}
\label{fig1_supp}
\end{figure*}

\begin{figure*}[t]
\centering
\includegraphics[width=1\linewidth]{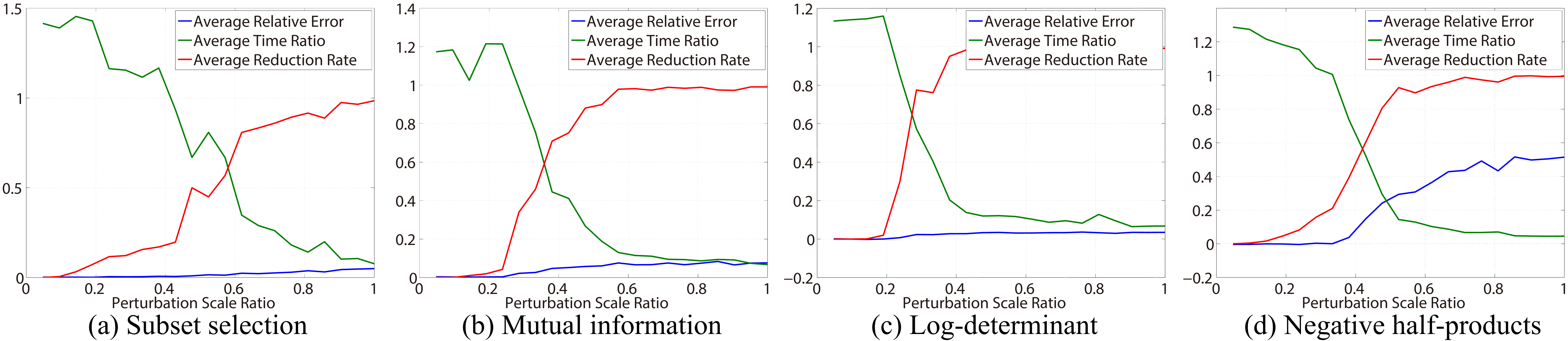}
\caption{Maximization Results Using Random Permutation \cite{iyer2013fast}}
\label{fig2_supp}
\end{figure*}

\begin{figure*}[t]
\centering
\includegraphics[width=1\linewidth]{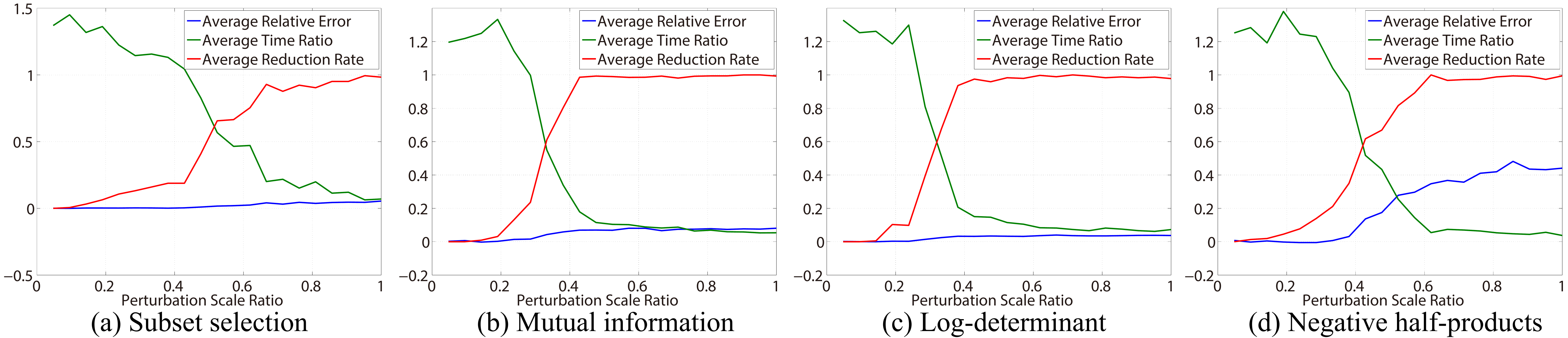}
\caption{Maximization Results Using Random Local Search \cite{iyer2013fast}}
\label{fig3_supp}
\end{figure*}

\begin{figure*}[t]
\centering
\includegraphics[width=1\linewidth]{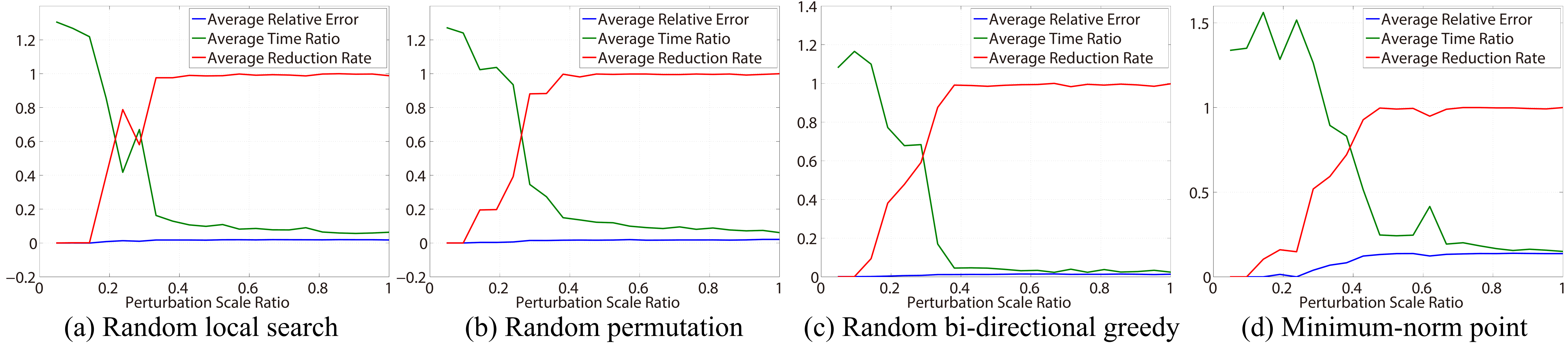}
\caption{Results of Log-determinant Function Using CIFAR Dataset}
\label{fig4_supp}
\end{figure*}

\end{document}